\documentclass[10pt,twocolumn,letterpaper]{article}

\usepackage{cvpr}              % To produce the CAMERA-READY version
\definecolor{cvprblue}{rgb}{0.21,0.49,0.74}
\usepackage[pagebackref,breaklinks,colorlinks,allcolors=cvprblue]{hyperref}
\usepackage{multirow}

\def\paperID{10} % *** Enter the Paper ID here
\def\confName{CVPR}
\def\confYear{2025}

\title{Is Temporal Prompting All We Need For Limited Labeled Action Recognition?}

\author{
Shreyank N Gowda\textsuperscript{1}\quad
Boyan Gao\textsuperscript{2}\quad
Xiao Gu\textsuperscript{2}\quad
Xiaobo Jin\textsuperscript{3}\\ \\
\textsuperscript{1}University of Nottingham\quad
\textsuperscript{2}University of Oxford\quad
\textsuperscript{3}Xi'an Jiaotong-Liverpool University
}

\begin{document}
\maketitle
\begin{abstract}
 Video understanding has shown remarkable improvements in recent years, largely dependent on the availability of large scaled labeled datasets. Recent advancements in visual-language models, especially based on contrastive pretraining, have shown remarkable generalization in zero-shot tasks, helping to overcome this dependence on labeled datasets. Adaptations of such models for videos, typically involve modifying the architecture of vision-language models to cater to video data. However, this is not trivial, since such adaptations are mostly computationally intensive and struggle with temporal modeling. We present TP-CLIP, an adaptation of CLIP that leverages temporal visual prompting for temporal adaptation without modifying the core CLIP architecture. This preserves its generalization abilities. TP-CLIP efficiently integrates into the CLIP architecture, leveraging its pre-trained capabilities for video data. Extensive experiments across various datasets demonstrate its efficacy in zero-shot and few-shot learning, outperforming existing approaches with fewer parameters and computational efficiency. In particular, we use just \emph{1/3} the GFLOPs and \emph{1/28} the number of tuneable parameters in comparison to recent state-of-the-art and still \emph{outperform} it by up to 15.8\% depending on the task and dataset.
\end{abstract}    
\section{Introduction}
\label{sec:intro}

Video understanding has significantly evolved recently~\cite{gowda2021smart,yang2020temporal,wang2021tdn,mazzia2022action}, fueled by computer vision and machine learning breakthroughs. Central to this progress are large language models (LLMs)~\cite{kenton2019bert} and visual-language models~\cite{clip,li2022blip,singh2022flava,alayrac2022flamingo}, showcasing potential across image-language tasks. Models like CLIP~\cite{clip} and its video adaptations~\cite{rao2022denseclip,wang2021actionclip,wasim2023vita} have been pivotal in advancing this field.

Historically, video understanding heavily relied on extensive labeled datasets~\cite{i3d,kinetics600}. However, this reliance shifted with the emergence of LLMs and visual-language models, showcasing strong zero-shot generalization capabilities, reducing the need for labeled data. Early efforts like VideoBERT~\cite{sun2019videobert} employed self-supervised learning from video-text data, validating pre-training and fine-tuning effectiveness across multiple video tasks. This approach evolved with diverse architectures and strategies~\cite{vificlip,wasim2023vita,ivl}.

\begin{figure}[t]
    \centering    \includegraphics[width=\linewidth]{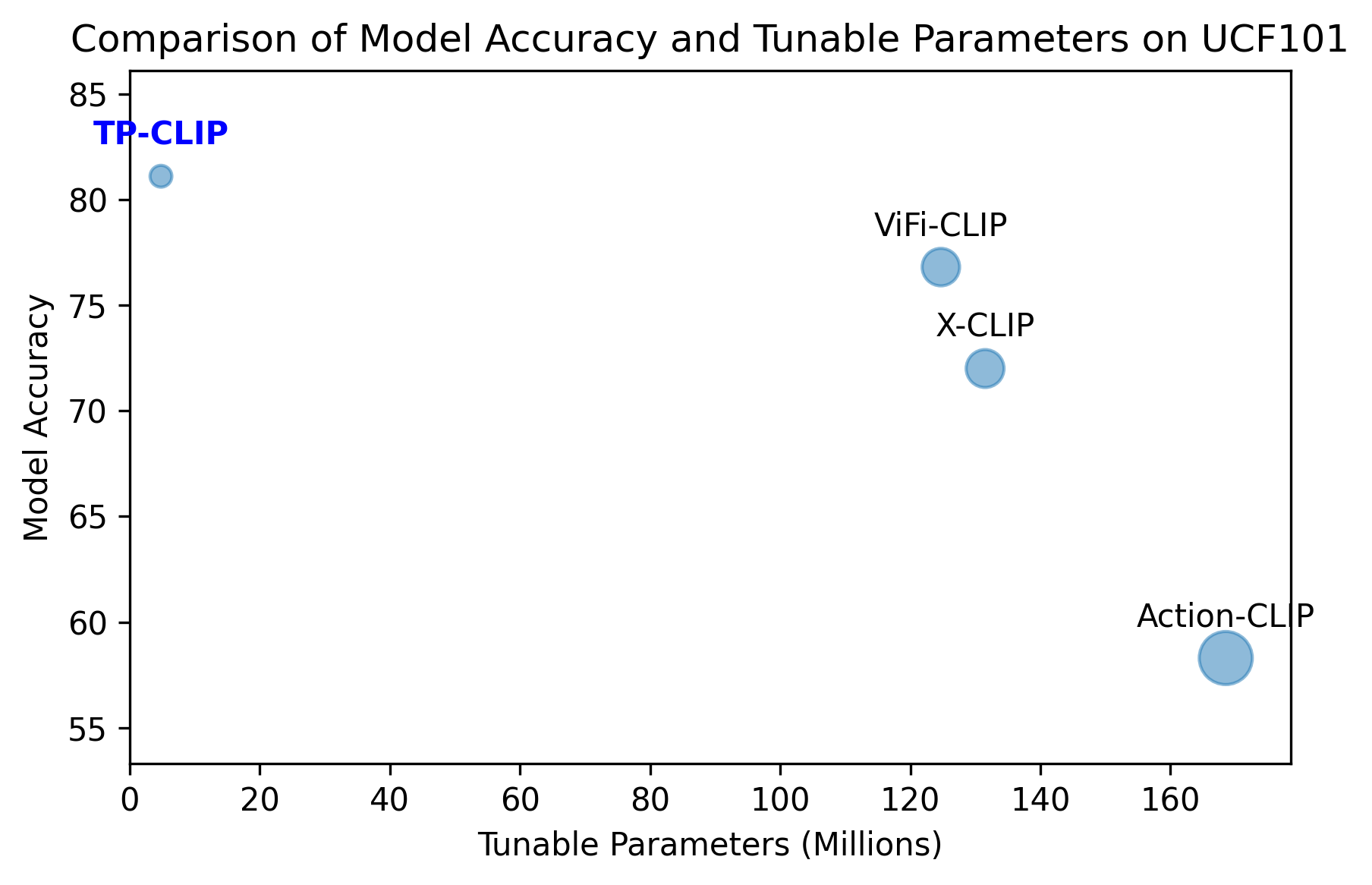}
\caption{Comparing CLIP-based models on UCF101: Our TP-CLIP (highlighted) excels with fewer parameters, surpassing recent SOTA models. Bubble size = model GFLOPS.}
    %\vspace{-0.6cm}
    \label{fig:teaser}
\end{figure}

Such advancements have been further accelerated by the emergence of LLMs, such as ChatGPT, by integrating extensive knowledge and underpinning contextual comprehension. This integration facilitates sophisticated multimodal models, adept at processing intricate interactions between visual and textual data, akin to human-like understanding. Recent research explores methodologies to enhance these models' video understanding capabilities, leveraging their extensive training on multimodal data to handle diverse tasks effectively.

Nevertheless, temporal modeling remains a key challenge in video understanding, especially for adaptations of image-language models seeking to leverage existing pre-trained capabilities. Various adaptations for video, like adding self-attention layers or employing dedicated video models, have been proposed~\cite{luo2022clip4clip,xu2021videoclip,pan2022st,wasim2023vita}. Prompt learning has emerged as a recent approach~\cite{wasim2023vita,ju2022prompting}, yet it requires full fine-tuning, potentially compromising model generalization. Prompted by this and the surge in prompt learning research, we inquire: Can we maintain CLIP's generalization more efficiently? Specifically, can temporal prompting alone suffice to adapt CLIP for videos?

Introducing TP-CLIP, we propose a method that seamlessly integrates temporal visual prompting into the CLIP architecture, thereby maintaining its generalization prowess. This approach efficiently handles video data, mitigating the computational intensity and complexities associated with temporal modeling. In Figure~\ref{fig:teaser}, TP-CLIP is compared with other CLIP-based video adaptations, showcasing significant savings in tunable parameters and GFLOPs, alongside accuracy improvements. This method upholds CLIP's core strengths while expanding its capabilities to tackle video data. Across various datasets, our experiments demonstrate TP-CLIP's effectiveness in zero-shot and few-shot learning, surpassing existing approaches in terms of parameter efficiency and computational effectiveness.
\section{Related Work}

\paragraph{Action Recognition.} Action recognition has evolved significantly, starting from early works that primarily focused on handcrafted features and traditional classifiers~\cite{dalal2005histograms,wang2013action,wang2013dense,jain2013better}. Progressively, the field shifted towards deep learning methods, with Convolutional Neural Networks, Recurrent Neural Networks (RNNs) and Long Short-Term Memory networks becoming a staple for extracting spatial and temporal features from video frames~\cite{simonyan2014two,tran2018closer,wang2016temporal,i3d}. With the advancements in Transformers~\cite{transformer}, video-based transformer models have become increasingly popular~\cite{timesformer,vivit,videoswin,li2022uniformer}. However, these models are extremely expensive to train and whilst efficient variants exist, there is much scope for improvement.

\paragraph{Vision-Language Models.} The domain of vision-language models represents a fascinating intersection of computer vision and natural language processing. Early works in this area primarily focused on image captioning~\cite{xu2015show}, where the goal was to generate descriptive text for images. This evolved into more complex tasks like visual question answering (VQA), requiring the model to understand both visual content and language queries. Recent breakthroughs include multimodal transformers~\cite{clip,li2022blip,alayrac2022flamingo,singh2022flava}, which integrate visual and textual inputs in a unified framework, enabling more sophisticated interactions between visual and textual data excelling in tasks ranging from zero-shot image classification to creative image generation based on textual descriptions. These advancements underscore the growing capability of AI models to understand and generate content across visual and linguistic domains. Whilst these models excel in generalization abilities, they are trained on millions and sometimes billions of data points to obtain their high level performance. This is not something feasible for small research labs and companies. Efficiently adapting these models to tasks with minimal parameter costs is an important research direction.

\paragraph{Adapting Vision-Language Models for Video.} 
Adapting vision-language models (VLMs) for video tasks is a dynamic research domain, merging visual and linguistic advancements for video analysis. "Tem-Adapter"\cite{temadapter} focuses on video question answering, emphasizing temporal dynamics and complex semantics through novel aligner components. LAVILA\cite{lavila} repurposes Large Language Models for video-language representation learning, enhancing temporal synchronization and alignment diversity. CLIP4Clip~\cite{luo2022clip4clip} tailors CLIP for video-language retrieval tasks by encoding frames for feature sequences. VideoCLIP~\cite{xu2021videoclip} utilizes contrastive pre-training to unify video and text understanding. Other methods include spatio-temporal adapters~\cite{pan2022st}, cross-frame communication layers~\cite{ju2022prompting}, dynamic dilated convolutions~\cite{fe-adapter} and dedicated video encoder modules~\cite{ni2022expanding}. ViFi-CLIP~\cite{vificlip} fine-tunes CLIP for enhanced temporal understanding in various video learning scenarios. While promising, these methods introduce computational overhead and require full fine-tuning, potentially compromising the model's generalization capability. Most similar to our work is EZ-CLIP~\cite{ahmad2023ez}, which also adapts CLIP for temporal understanding in videos through prompt-based methods. Unlike EZ-CLIP which relies on a separate motion loss and MHA-based prompt processing, our TP-CLIP introduces a dedicated temporal encoder with 1D convolutional layers that directly incorporates temporal context into visual prompts while requiring fewer GFLOPs and tunable parameters.

\paragraph{Temporal Prompting.} 
Prompting efficiently utilizes pre-trained models for new tasks without retraining core parameters. It adds a few learnable tokens to the model's input, balancing generalization with task performance. Initially for NLP, prompting extends to computer vision, aiding vision-language models for video recognition. VPT~\cite{vpt} and UPT~\cite{upt} adjust minimal parameters for impressive task results. Prompt tuning achieves similar outcomes to full fine-tuning with fewer changes. IV-L~\cite{ivl} adapts CLIP for video with text prompts, potentially impacting generalization. Vita-CLIP~\cite{wasim2023vita} optimizes video classification with multimodal prompt learning, emphasizing zero-shot performance. However, it adds complexity with multiple prompting layers. We focus solely on temporal prompting for efficient CLIP performance without sacrificing generalization.

\paragraph{Limited Labeled Action Recognition.} Previous research focuses on establishing a shared embedding space between video features and semantic labels~\cite{xu2016multi,xu2017transductive}. More recent methods leverage out-of-distribution detectors~\cite{od} and graph neural networks~\cite{gao2019know}. Clustering-based techniques like CLASTER~\cite{claster} emphasize joint modeling of visual-semantic features, while methods like ReST~\cite{rest} and JigSawNet~\cite{jigsaw} bridge visual and text spaces through various mechanisms. With the rise of foundation models, image-based models have demonstrated excellent zero-shot capabilities when adapted to video tasks. These adaptations include multi-modal prompting~\cite{wasim2023vita} or self-regulating prompts~\cite{srp}. Semantic embedding enhancements, from simple word2vec to elaborate definitions like action ``Stories''~\cite{stories}, have shown great promise. For GZSL settings, approaches include using CVAEs to disentangle visual features\cite{chen2021semantics}, combining feature generation with contrastive embeddings~\cite{han2021contrastive}, and employing adversarial training with augmented samples~\cite{chen2023zero}. More recently, GIL~\cite{gowda2024continual} proposed ideas of continual learning to help foundational models overcome catastrophic forgetting. Unlike these methods that often require complex architectural modifications or extensive computational resources, our approach focuses on efficient temporal adaptation through visual prompting, maintaining CLIP's strong generalization abilities while significantly reducing parameter count and computational requirements.

\section{Methodology}

\begin{figure*}[t]
    \centering    \includegraphics[width=0.7\linewidth]{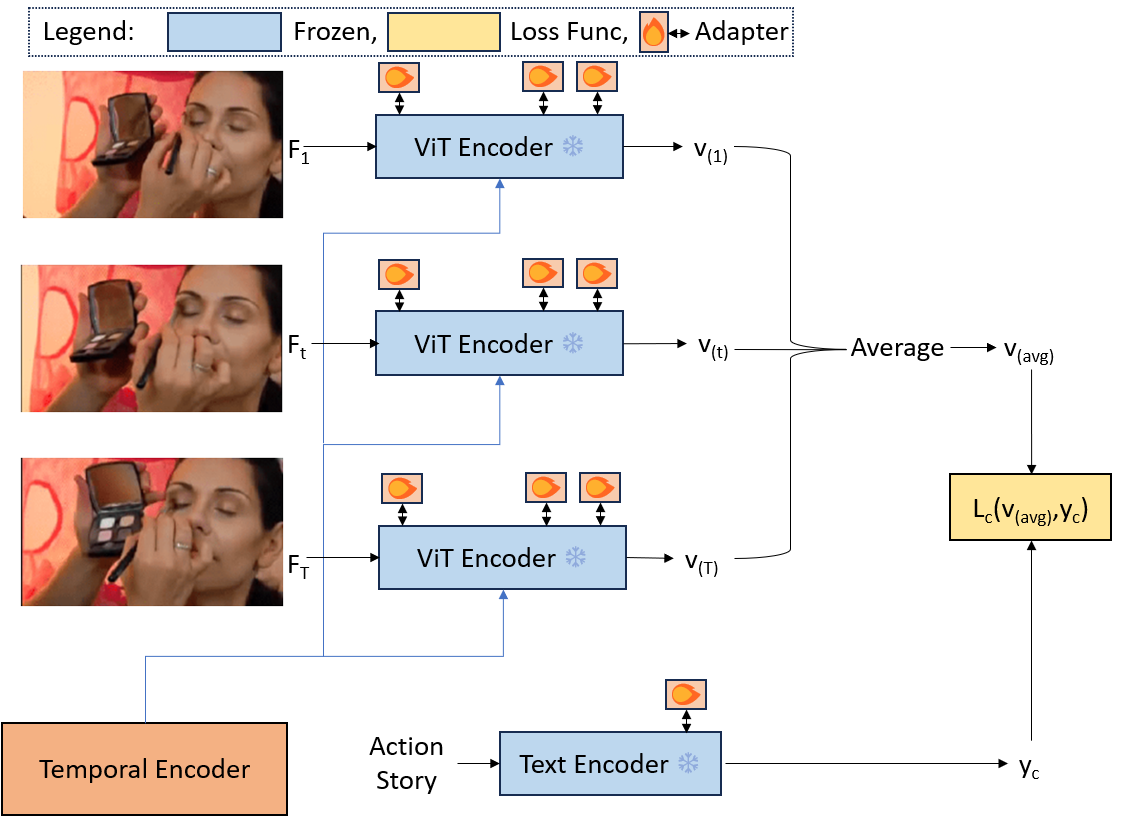}
\caption{We add temporal prompt learning to CLIP. The temporal visual prompts (using the Temporal Encoder) are instrumental in identifying inter-frame relationships, which aids in modeling the aspect of motion in videos. Additionally, adapter modules are utilized to modify spatial features, facilitating more effective temporal learning. }
    %\vspace{-0.6cm}
    \label{fig:overview}
\end{figure*}

TP-CLIP, our proposed method, is designed to modify pretrained image-based vision-language models for video application. This is achieved through a straightforward temporal visual prompting strategy, which is intended to preserve the robust generalization abilities of CLIP without necessitating major architectural alterations to the model. A visual summary of this methodology is presented in Figure~\ref{fig:overview}. This enhancement significantly improves efficiency while maintaining computational effectiveness. We explain the model in more detail in the following subsections.

\subsection{Problem Statement}

Given a collection of video samples, we represent a video as $V = \{f(1), f(2), \ldots, f(T)\}$ $ \in $ $ \mathbb{R}^{T \times H \times W \times C}$ consisting of $T$ frames each with a resolution of $H \times W$ and $C = 3$ for RGB channels. Each video is associated with a text prompt $c$ which represents the target action class. We seek a model $M_v$ that can leverage visual encoders $E_{image}$ and text encoder $E_{text}$ from image-based pre-training and provide encodings for a video $V$ and its text pair $c$ which are similar to each other and dissimilar from prompts of other action classes.

\textbf{Few-Shot Learning:} In this scenario, the model's ability to learn with minimal supervision is crucial. Consider a dataset \( \mathcal{D_{S}} \) with labels \( \mathcal{Y_{S}} = \{y_i\}_{i=0}^k \). A K-shot dataset is constructed by randomly selecting \( K \) examples from each category \( y_i \in \mathcal{Y_{S}} \) for training purposes. We experiment with various values of \( K \), namely \( K = 2, 4, 8, \) and \( 16 \), in line with the methodology used in Vifi CLIP~\cite{wasim2023vita}. Evaluation is performed on the validation subset of \( \mathcal{D_{S}} \).

\textbf{Zero-Shot Learning:} In zero-shot learning, the model is evaluated on its capability to correctly infer information about unseen classes. Let \( \mathcal{Z_{S}} \) represent the set of unseen classes, where \( \mathcal{Z_{S}} \cap \mathcal{Y_{S}} = \emptyset \). The model, trained on \( \mathcal{Y_{S}} \), is tested on \( \mathcal{Z_{S}} \) to assess its predictive accuracy without prior exposure to these classes.

\textbf{Base-to-Novel Class Generalization:} To evaluate the model's generalization to new classes, we adopt the base-to-novel setting used in Vifi CLIP~\cite{vificlip}. The dataset \( D_{S} \) with labels \( Y_{S} = \{y_i\}_{i=0}^{k} \) is divided into base classes \( Y_{B} \) and novel classes \( Y_{N} \). This ensures \( Y_{B} \cup Y_{N} = Y_{S} \) and \( Y_{B} \cap Y_{N} = \emptyset \). The model is initially trained on the base classes \( Y_{B} \) and subsequently tested on both the base \( Y_{B} \) and novel classes \( Y_{N} \), to determine its ability to adapt to and learn from new class data.

\subsection{TP-CLIP}

\subsubsection{Overview}

A visual-language model utilizes two encoders: an image encoder \( E_{\text{image}} \) and a text encoder \( E_{\text{text}} \). These encoders process images and text respectively, and are trained on image-text pairs to maximize encoding similarity using a contrastive objective as outlined in CLIP~\cite{clip}. The contrastive loss function is defined as:

\begin{equation}
    \mathcal{L}_{con} = \sum_{i=1}^{N} -\log \frac{\exp(\text{sim}(e_i, y_i) / \tau)}{\sum_{j=1}^{N} \exp(\text{sim}(e_i, y_j) / \tau)}
\end{equation}

where \( \text{sim}(e, y) \) denotes similarity between encodings, \( \tau \) is the temperature parameter, and \( N \) the number of instances. For zero-shot tasks like image classification, encodings are compared to textual class descriptions, leveraging the model's capability to understand textually prompted classes. Ju et al~\cite{ivl} illustrate how such adaptability in visual-language models, including trainable prompts and adapters, allows significant model enhancements without altering core architecture.

In the context of video, the model processes a sequence of frames. Each frame is encoded by \( E_{\text{image}} \) to produce embeddings \(\{e(1), e(2), \ldots, e(T)\} \in \mathbb{R}^{T \times D}\), with \( e(t) \) representing the \( t \)-th frame's embedding. A naive way of proceeding is to use frame-wise contrastive learning whilst maintaining CLIP's architecture. We modify this in the proposed TP-CLIP. The frame embeddings are concatenated with learnable temporal visual prompts \( T_{E}(F_{1:T}) \) to form \( v(t) \), which assists in modeling motion by capturing inter-frame relations. Averaging these produces a video-level representation \( v_{\text{avg}} \in \mathbb{R}^D \). The model also includes trainable components such as temporal and spatial adapters, and text prompts, enhancing its robustness and generalization ability, while fundamental weights from CLIP remain frozen.

\subsubsection{Temporal Visual Prompting}

Temporal Visual Prompting (TVP) in TP-CLIP aims to extend the capabilities of static image-based visual-language models to effectively process the dynamic and temporal aspects inherent in video content.

\textbf{Design Principle:} The primary objective of TVP is to integrate temporal dynamics seamlessly into the pre-existing framework of CLIP-based models. This integration is achieved without altering the fundamental architecture. Temporal prompts, which encapsulate dynamic video features, are generated and then concatenated with the visual encodings of individual frames. This results in a composite representation that marries spatial and temporal information effectively.

\begin{equation} 
    v(t) = [E_{\text{image}}(F_t) ; T_E(F_{1:T})], \quad \forall t \in \{1, \ldots, T\}
\end{equation}

Here, $F_{1:T}$ denotes the sequence of frames, $E_{\text{image}}$ is the image encoder, $T_E$ represents the Temporal Encoder, and $v(t)$ is the combined spatial-temporal encoding for frame $t$.

\textbf{Temporal Encoding:} The Temporal Encoder $T_E$ processes the frame-wise features to produce a temporal context vector $T_{\text{context}}$, which captures the essence of motion and change over time.

The temporal information is synthesized using a 1D convolutional layer followed by a temporal embedding layer. The convolutional layer detects local temporal patterns, which the embedding layer then distills into a compact representation. Each frame is passed through the pre-trained image encoder of CLIP and concatenated to form a long vector representation. This vector representation is passed through the 1D convolutional layer, followed by a fully connected layer and a non-linearity.

\begin{equation}
    V_{\text{conv}} = \text{Conv1D}(E_{\text{image}}(F_{1}):E_{\text{image}}(F_{2})..:E_{\text{image}}(F_{T}) 
\end{equation}
\begin{equation}
    T_{\text{context}} = \text{ReLU}(\text{FC}(V_{\text{conv}}))
\end{equation}

\textbf{Integration with CLIP:} Integration of the temporal visual prompts into the CLIP model is executed such that each video frame's spatial encoding, generated by $E_{\text{image}}$, is enriched with the temporal context from $T_E$. This enriched encoding comprehensively represents both the spatial and temporal features of the video.

\textbf{Contrastive Learning with Temporal Context:} The contrastive learning framework of CLIP is adapted to include the temporal dimension. The modified loss function aligns these enriched video encodings with the corresponding textual descriptions that reflect both visual and temporal aspects of the content.

\begin{equation}
    \mathcal{L}_{c,t} = -\sum_{i=1}^{N} \log \frac{\exp(\text{sim}(v_{(i)_{avg}}, y_i) / \tau)}{\sum_{j=1}^{N} \exp(\text{sim}(v_{(i)_{avg}}, y_j) / \tau)}
\end{equation}

In this setup, $v_{(i)_{avg}}$ denotes the average encoding for video $i$, and $y_i$ represents the encoded textual description of the class~\cite{stories} corresponding to video class $i$. The parameter $\tau$ serves as a temperature control to fine-tune the softmax distribution.

\begin{figure}[t]
    \centering    \includegraphics[width=0.99\linewidth]{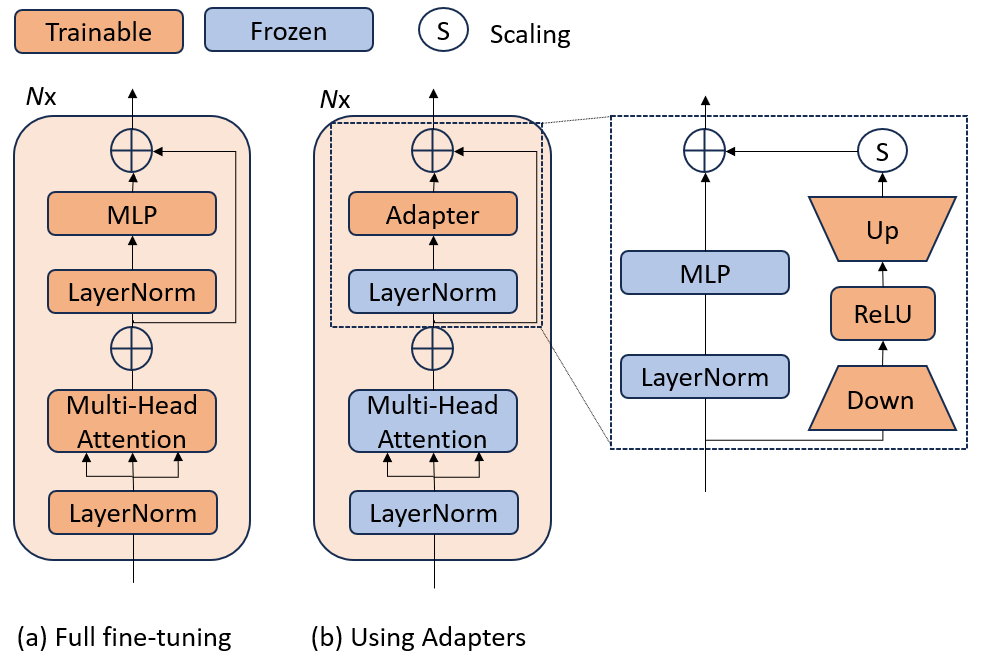}
\caption{Adapters ensure minimal parameter updates whilst keeping the generalization ability of the pre-trained model consistent. In comparison to (a) full fine-tuning, using (b) adapter modules significantly reduces tuneable parameter count.}
    %\vspace{-0.6cm}
    \label{fig:adapters}
\end{figure}

\textbf{Tuneable Parameters:} The TP-CLIP architecture is designed to be highly efficient in terms of parameter usage. The model leverages tuneable adapters and the Temporal Encoder within its pipeline, allowing significant enhancements without extensive retraining or modification of the core CLIP architecture. Figure~\ref{fig:adapters} explain how adapter modules are integrated. These small layers ensure reduced parameter fine-tuning whilst keeping the original parameters frozen.

\section{Experimental Analysis}

\subsection{Datasets}

Our proposed method is tested on five video action recognition benchmarks: Kinetics-400~\cite{i3d}, Kinetics-600~\cite{kinetics600}, HMDB-51~\cite{hmdb}, UCF-101~\cite{ucf101}, and Something Something V2 (SSv2)~\cite{ssv2}. Kinetics-400, Kinetics-600, HMDB-51, and UCF-101 exhibit some appearance biases, with the background playing a significant role in action recognition. Conversely, SSv2 presents a more challenging scenario where temporal comprehension is essential.

\subsection{Evaluation Settings}

In the zero-shot evaluation scenario, models pretrained on the Kinetics-400 dataset are tested on three distinct datasets: HMDB-51, UCF-101, and Kinetics-600. The evaluation on HMDB-51 and UCF-101 is conducted over the three provided validation splits for each, with the results presented as the average top-1 accuracy. For Kinetics-600, models are evaluated on a subset of 220 action categories that are distinct from those in Kinetics-400, and the top-1 accuracy is noted. An 8-frame, single-view inference approach is employed during testing. To prevent any pre-training bias, only those classes from the 400 that do not have a counterpart in the evaluation datasets are utilized for pre-training, as suggested in TruZe~\cite{gowda2021new}.

In the few-shot learning framework, a standard K-shot division is formulated, utilizing K instances in alignment with the partitioning methodology established in Vifi CLIP~\cite{vificlip}. We conduct experiments with 2, 4, 8, and 16-shot configurations across three datasets: HMDB-51, UCF-101, and SSv2. Evaluation of the models is carried out on the initial validation split for both HMDB-51 and UCF-101, while the complete validation set is used for SSv2, which is known for its temporal complexity in the actions presented.

To effectively evaluate the generalization capabilities of different methodologies, we implement the base-to-novel generalization framework as described in Vifi CLIP~\cite{vificlip} for video action recognition tasks. In this setup, a model is initially trained on a set of base (seen) classes using a few-shot approach and is then tested on a distinct set of novel (unseen) classes. Our analysis extends across four datasets: Kinetics-400, HMDB-51, UCF-101, and SSv2. Each dataset is divided into three training splits, with the total categories evenly split into two groups. The base classes are formed from the most frequently appearing categories, while the less common ones are categorized as novel classes. The evaluation process utilizes an 8-frame, single-view inference methodology.

\subsection{Implementation Details}

In our study, we evaluated our model's generalization capabilities using not only the CLIP ViT-B/16 model but also the CLIP ViT-B/32 and ViT-H/14 architectures. Our methodology involves the sparse sampling of only 8 frames per video to maintain computational efficiency. The optimization process utilizes the AdamW optimizer, starting with a learning rate of  8e-6 and a weight decay of 0.001 spanning over 50 epochs. The training was conducted on a single NVIDIA A100 80GB GPU, accommodating a batch size of 64, and upheld an input frame resolution of 224 × 224. As for text input, we leveraged a pretrained 12-layer BERT model, specifically designed for the CLIP B/16 variant, which includes an embedding size of 512 and a context length of 77. For the text encoder's inputs, we utilized action descriptions generated by a Large Language Model (LLM), sourced from the dataset referred to as Stories~\cite{stories}.

\subsection{Dataset Details}

The Kinetics-400 and Kinetics-600 datasets feature 400 and 600 human action classes from YouTube videos, with approximately 240,000 and 410,000 training clips, respectively, and 20,000 and 29,000 validation clips. The HMDB-51 dataset includes 51 action categories from various sources, with 3,570 training and 1,530 validation videos, further divided into three splits. UCF-101 consists of 13,000 YouTube videos across 101 action categories, with a standard split of 9,537 training and 3,783 validation videos. Finally, the Something Something V2 (SSv2) dataset is an extensive collection of 174 categories of fine-grained actions involving human-object interactions, comprising 168,913 training and 24,777 validation videos, with a focus on top-1 accuracy over the validation split.

\subsection{Zero-shot Performance}

We first train TP-CLIP using 8 frames on the K400 training set, followed by zero-shot evaluations on UCF101, HMDB51, and K600. The pre-training of our model is on Kinetics-400 classes that don't overlap with the test datasets, as detailed in E2E~\cite{e2e}. We substitute class-specific context with tokenized class descriptions. Our zero-shot results on UCF101, HMDB51, and K600 are displayed in Table~\ref{tab:zsl} and Table~\ref{tab:zslk600}. As shown in Table~\ref{tab:zsl}, our method surpasses previous approaches by 1.4\% and 5.6\% on HMDB51 and UCF101, respectively. In Table~\ref{tab:zslk600}, we also set a new benchmark in zero-shot performance on K600, exceeding the prior best by 2.3\%. This success is attributed to our model's efficient learning of appearance and motion, reducing computational load and rendering the model more streamlined for training and deployment. Additionally, our model retains CLIP's generalization capabilities promoting further research into the idea of temporal prompting.

\begin{table}[htbp]
    \centering
    \begin{tabular}{lcc}
        \hline
        Method & HMDB-51 & UCF-101 \\
        \hline
        \multicolumn{3}{l}{\textit{Methods with Vision Training}} \\
        E2E~\cite{e2e} & 32.7 & 48 \\
        ER-ZSAR~\cite{er}  & 35.3 $\pm$ 4.6 & 51.8 $\pm$ 2.9 \\
        SPOT~\cite{spot} & 35.9 $\pm$ 2.5 & 40.9 $\pm$ 2.6 \\
        CLASTER~\cite{claster} & 42.6 $\pm$ 2.6 & 52.7 $\pm$ 2.2 \\
        JigSawNet~\cite{jigsaw}  & 39.3 $\pm$ 3.9 & 56.8 $\pm$ 2.8 \\
        \hline
        \multicolumn{3}{l}{\textit{Methods with Vision-Language Training}} \\
        A5~\cite{ivl}  & 44.3 $\pm$ 2.2 & 69.3 $\pm$ 4.2 \\
        X-CLIP~\cite{ni2022expanding}  & 44.6 $\pm$ 5.2 & 72.0 $\pm$ 2.3 \\
        X-Florence~\cite{ni2022expanding}  & 48.4 $\pm$ 4.9 & 73.2 $\pm$ 4.2 \\
        Vita-CLIP~\cite{wasim2023vita} & 48.6 $\pm$ 0.6 & 75.0 $\pm$ 0.6 \\
        SDR-CLIP~\cite{stories}  & 52.7 $\pm$ 3.4 & 75.5 $\pm$ 3.2 \\
        GIL~\cite{gowda2024continual} & 53.9 $\pm$ 1.4 & 79.4 $\pm$ 1.4 \\
        EZ-CLIP~\cite{ahmad2023ez} & 52.9 $\pm$ 1.6 & 79.4 $\pm$ 1.8 \\
        \textbf{TP-CLIP (Ours)} & \textbf{54.1 $\pm$ 1.2} & \textbf{81.1 $\pm$ 1.2} \\
        \hline
    \end{tabular}
    \caption{Comparison of zero-shot performances on HMDB51 and UCF101 against state-of-the-art.}
    \label{tab:zsl}
\end{table}

\begin{table}[htbp]
    \centering
    \begin{tabular}{lcc}
        \hline
        Method & Top-1  & Top-5\\
        \hline
        \multicolumn{3}{l}{\textit{Methods with Vision Training}} \\
        ER-ZSAR~\cite{er}  & 42.1 $\pm$ 1.4 & 73.1 $\pm$ 0.3 \\
        JigSawNet~\cite{jigsaw} & 45.9 $\pm$ 1.6 & 78.8 $\pm$ 1.0 \\
        \hline
        \multicolumn{3}{l}{\textit{Methods with Vision-Language Training}} \\
        A5~\cite{ju2022prompting} & 58.4 $\pm$ 1.1 & 82.6 $\pm$ 1.6 \\
        X-CLIP~\cite{ni2022expanding}  & 65.2 $\pm$ 0.4 & 86.1 $\pm$ 0.8 \\
        X-Florence~\cite{ni2022expanding} & 68.8 $\pm$ 0.9 & 88.4 $\pm$ 0.6 \\
        Vita-CLIP~\cite{wasim2023vita}  & 67.4 $\pm$ 0.5 & 86.9 $\pm$ 0.6 \\
        SDR+CLIP~\cite{stories}  & 55.1  $\pm$ 2.2 & 86.1 $\pm$ 3.1  \\
        \textbf{TP-CLIP (Ours)} & \textbf{71.1 $\pm$ 1.2} & \textbf{92.4 $\pm$ 1.1} \\
        \hline
    \end{tabular}
    \caption{Comparison of zero-shot performances on Kinetics200 against state-of-the-art.}
    \label{tab:zslk600}
\end{table}

\subsection{Few-shot Performance}

In the few-shot learning scenario, our focus is on the model's ability to adapt with minimal supervision. Consider a dataset $\mathcal{D_{S}}$ with a set of labels $\mathcal{Y_{S}} = \{y_i\}_{i=0}^k$. We construct a general K-shot dataset by randomly selecting $K$ examples from each category $y_i \in \mathcal{Y_{S}}$ for training purposes. Our experiments are conducted with various values of $K$, namely $K = 2, 4, 8,$ and $16$, aligning our samples with those used in Vita CLIP~\cite{wasim2023vita}. The evaluation is carried out on the validation subset of $\mathcal{D_{S}}$. Table~\ref{tab:few-shot} presents the performance of TP-CLIP in comparison to other methods that adapt CLIP for video tasks. Notably, although TP-CLIP does not demonstrate substantial performance enhancements due to its lower count of tunable parameters, it consistently shows robust performance across different shot configurations, outperforming other methods in all scenarios by up to 2.6\%.

\begin{table*}[]
    \centering
    \centering
\resizebox{\textwidth}{!}{
\begin{tabular}{ccccccccccccc}
\hline
 \multirow[b]{2}{*}{ Method } & \multicolumn{4}{c}{ HMDB-51 } & \multicolumn{4}{c}{ UCF-101 } & \multicolumn{4}{c}{ SSv2 } \\
 & $\mathrm{K}=2$ & $\mathrm{~K}=4$ & $\mathrm{~K}=8$ & $\mathrm{~K}=16$ & $\mathrm{~K}=2$ & $\mathrm{~K}=4$ & $\mathrm{~K}=8$ & $\mathrm{~K}=16$ & $\mathrm{~K}=2$ & $\mathrm{~K}=4$ & $\mathrm{~K}=8$ & $\mathrm{~K}=16$ \\

\hline ActionCLIP~\cite{wang2021actionclip} & 47.5 & 57.9 & 57.3 & 59.1 & 70.6 & 71.5 & 73.0 & 91.4 & 4.1 & 5.8 & 8.4 & 11.1 \\
XCLIP~\cite{ni2022expanding} & 53.0 & 57.3 & 62.8 & 62.4 & 71.4 & 79.9 & 83.7 & 91.4 & 3.9 & 4.5 & 6.8 & 10.0 \\
 A5~\cite{ju2022prompting} & 39.7 & 50.7 & 57.0 & 62.4 & 71.4 & 79.9 & 85.7 & 89.9 & 4.4 & 5.1 & 6.1 & 9.7 \\
 ViFi CLIP~\cite{vificlip} & 57.2 & 62.7 & 64.5 & 66.8 & 80.7 & 85.1 & 90.0 & 92.7 & 6.2 & 7.4 & 8.5 & 12.4 \\
 EZ-CLIP~\cite{ahmad2023ez} & 57.3 & 61.1 & 65.4 & 67.7 & 84.7 & 88.3 & 91.3 & 92.8 & 6.8 & 8.7 & 9.6 & 13.2 \\
 \textbf{TP-CLIP (Ours)} & \textbf{58.2} & \textbf{63.4} & \textbf{66.6} & \textbf{68.1} & \textbf{83.3} & \textbf{87.1} & \textbf{91.1} & \textbf{93.1} & \textbf{7.8} & \textbf{9.3} & \textbf{10.5} & \textbf{13.6} \\
\hline
\end{tabular}}
    \caption{Comparing TP-CLIP to current methods on HMDB-51, UCF-101, and SSv2 datasets in various few-shot scenarios (K = 2, 4, 6, 8 shots), using top-1 accuracy for assessment. TP-CLIP shows strong generalization skills and enhanced performance in most tests, despite having fewer adjustable parameters.}
    \label{tab:few-shot}
\end{table*}

\subsection{Base-to-Novel Performance}

\begin{table*}[]
    \centering
    \centering
\resizebox{\textwidth}{!}{
\begin{tabular}{lcccccccccccc}
\hline \multirow{2}{*}{ Method } & \multicolumn{3}{c}{ Kinetics-400 } & \multicolumn{3}{c}{ HMDB-51 } & \multicolumn{3}{c}{ UCF-101 } & \multicolumn{3}{c}{ SSv2 } \\
 & Base & Novel & HM & Base & Novel & HM & Base & Novel & HM & Base & Novel & HM \\
\hline ActionCLIP~\cite{wang2021actionclip} & 61.0 & 46.2 & 52.6 & 69.1 & 37.3 & 48.5 & 90.1 & 58.1 & 70.7 & 13.3 & 10.1 & 11.5 \\
 XCLIP~\cite{ni2022expanding} & 74.1 & 56.4 & 64.0 & 69.4 & 45.5 & 55.0 & 89.9 & 58.9 & 71.2 & 8.5 & 6.6 & 7.4 \\
 A5~\cite{ju2022prompting} & 69.7 & 37.6 & 48.8 & 46.2 & 16.0 & 23.8 & 90.5 & 40.4 & 55.8 & 8.3 & 5.3 & 6.4 \\
 ViFi CLIP~\cite{vificlip} & 76.4 & 61.1 & 67.9 & 73.8 & 53.3 & 61.9 & 92.9 & 67.7 & 78.3 & 16.2 & 12.1 & 13.9 \\
EZ-CLIP~\cite{ahmad2023ez} & \textbf{81.2} & 60.0 & 69.0 & \textbf{78.5} & \textbf{57.8} & \textbf{66.5} & 95.6 & \textbf{76.4} & \textbf{84.9} & 45.1 & 18.5 & 26.2 \\
 \textbf{TP-CLIP (Ours)} & 81.1 & \textbf{62.7} & \textbf{70.1} & 77.7 & 57.1 & 65.8 & \textbf{96.2} & 71.4 & 82.0 & \textbf{53.9} & \textbf{20.5} & \textbf{29.7}\\
 \hline
  EZ-CLIP~\cite{ahmad2023ez} & 83.7 & 62.5 & 71.5 & 79.8 & \textbf{62.2} & 69.9 & 95.9 & 76.5 & 85.1 & 54.0 & 20.6 & 29.8 \\
  \textbf{TP-CLIP (Ours)} & \textbf{84.1} & \textbf{62.8} & \textbf{71.8} & \textbf{80.3} & 61.8 & \textbf{70.1} & \textbf{95.9} & \textbf{76.7} & \textbf{85.2} & \textbf{56.6} & \textbf{23.2} & \textbf{32.4} \\
\hline
\end{tabular}}
    \caption{Generalization from Base to Novel Classes: We conduct evaluations across four varied datasets - Kinetics400, HMDB-51, UCF-101, and SSv2, assessing models using top-1 accuracy. The harmonic mean (HM) is used to measure the balance between base and novel class performances. The last 2 results of EZ-CLIP and ours use the same motion loss in EZ-CLIP~\cite{ahmad2023ez}}
    \label{tab:base-novel}
\end{table*}

To investigate the adaptability of TP-CLIP to new classes, we follow the base-to-novel setting used in Vita CLIP~\cite{vificlip} for evaluation. Consider a dataset $D_{S}$ having labels \( Y_{S} = \{y_i\}_{i=0}^{k} \), which is divided into two segments: base classes \( Y_{B} \) and novel classes \( Y_{N} \). This division guarantees that \( Y_{B} \cup Y_{N} = Y_{S} \) and \( Y_{B} \cap Y_{N} = \emptyset \). The model undergoes training on base classes and is then tested on both base and novel classes. Table~\ref{tab:base-novel} illustrates the evaluation on four varied datasets: K-400, HMDB-51, UCF-101, and SSv2. TP-CLIP surpasses other models, demonstrating notable enhancements across these datasets. Enhancements up to 37.7\% in base classes and 8.3\% in novel classes are observed consistently across all datasets.

\subsection{Generalized Zero-shot Performance}

For the sake of complete evaluation, we also consider the case of generalized zero-shot action recognition. This is a more challenging domain where at test time, classes from both seen and unseen are evaluated. Very limited methods evaluate on this setting. However, we do this to show how much the proposed method improves on current state-of-the-art.

\begin{table}[htbp]

\centering
\begin{tabular}{lcc}
\toprule
 & HMDB-51 & UCF-101 \\
\midrule
% \multicolumn{3}{l}{\textit{Methods with Vision Training}} \\
WGAN~\cite{clswgan}  & 32.7 $\pm$ 3.4  & 44.4 $\pm$ 3.0 \\
OD~\cite{od}  & 36.1 $\pm$ 2.2  & 49.4 $\pm$ 2.4 \\
%GGM~\cite{ggm} (\textit{WACV'19})  & 20.1 $\pm$ 2.1  & 23.7 $\pm$ 1.2\\
CLASTER~\cite{claster}~\cite{er}  & 50.8 $\pm$ 2.8  & 52.8 $\pm$ 2.1 \\
%\toprule
% \multicolumn{3}{l}{\textit{Methods with Vision-Language Training}} \\
%Vision-Language Methods & Backbone & HMDB-51 & UCF-101 \\ \midrule
\textbf{TP-CLIP (Ours)}  & \textbf{55.1 $\pm$ 1.9} & \textbf{68.5 $\pm$ 1.2} \\
\bottomrule
\end{tabular}
\caption{Comparing TP-CLIP to SOTA in the generalized zero-shot setting on HMDB-51 and UCF-101. Reported values show the harmonic mean of the seen and unseen accuracies over 10 runs, $\pm$ $std$.}
\label{tab:gzsl}
\end{table}

In order to better analyze performance of the model on GZSL, we report the average seen and unseen accuracies along with their harmonic mean. The results on the UCF101~\cite{ucf101} and HMDB51~\cite{hmdb} datasets are reported in Table~\ref{tab:gzsl_vid}. Results are averaged from 10 runs, where for each run we create a different random train/test split (which is the same for all models).

\begin{table}[htb]
\begin{center}
\resizebox{0.45\textwidth}{!}{%
\begin{tabular}{ *{7}{c} }
\toprule
Model   & \multicolumn{3}{c}{HMDB51}& \multicolumn{3}{c}{UCF-101}\\
\midrule
&  u & s & H & u & s & H \\
\midrule
WGAN \cite{clswgan}  & 23.1 & 55.1 & 32.5 & 20.6 & 73.9 & 32.2\\
OD \cite{od}  & 25.9 & 55.8 & 35.4 & 25.3 & 74.1 & 37.7\\
CLASTER \cite{claster} & 43.7 & 53.3 & 48.0 & 40.8 & 69.3 & 51.3\\
%SDR \cite{stories}  & 50.1 & 57.5 & 53.5 & 47.3 & 81.2 & 59.7\\
\textbf{TP-CLIP} (Ours) & \textbf{52.8} & \textbf{57.8} & \textbf{55.1} & \textbf{59.2} & \textbf{80.1} & \textbf{68.5} \\
\bottomrule
\end{tabular}
}
\caption{Generalized Zero-Shot results, where `u', `s' and `H' correspond to average unseen accuracy, average seen accuracy and the harmonic mean of the two. All the reported results are on the same splits.} \label{tab:gzsl_vid}
\end{center}
\end{table}

\subsection{Comparing Efficiency Parameters}

We examine the computational complexity of different approaches as shown in Table~\ref{tab:efficiency}. Current CLIP adaptation techniques display reduced throughput, due to the use of video-specific learnable components to the standard CLIP model. TP-CLIP, on the other hand, tackles this efficiency issue effectively. It employs temporal visual prompts, consisting of learnable inputs integrated into the input space, to maintain temporal consistency. Consequently, TP-CLIP achieves greater efficiency in throughput while preserving similar FLOPs (Floating Point Operations Per Second) compared to earlier methods.

\begin{table*}
\centering
\begin{tabular}{lcccc}
\hline Method & GFLOPs $\downarrow$ & TP $\uparrow$ & Param (M) $\downarrow$ & T Param (M) $\downarrow$ \\
\hline ActionCLIP~\cite{wang2021actionclip} & 563 & 67.7 & 168.5 & 168.5 \\
XCLIP~\cite{ni2022expanding} & 287 & 58.5 & 131.5 & 131.5 \\
 ViFi CLIP~\cite{vificlip} & 281 & 71.1 & 124.7 & 124.7 \\
 EZ-CLIP~\cite{ahmad2023ez} & 102 & 322.4 & 87.5 & 5.2 \\
 \textbf{TP-CLIP (Ours)} & \textbf{94} & \textbf{337.1} & \textbf{81.2} & \textbf{4.4} \\
\hline
\end{tabular}
\caption{Evaluating TP-CLIP's efficiency against current methods. We measure throughput per view (TP) on a single A100 GPU. TP-CLIP excels in GFLOPs, throughput, and parameter count, showcasing its computational benefits. 'T Param (M)' refers to the number of tunable parameters.}
\label{tab:efficiency}
\end{table*}

\subsection{Results on TruZe~\cite{gowda2021new}}

We also evaluate on the more challenging TruZe split. Whilst our pre-training dataset does not have any overlapping classes with any of the testing datasets, we run on this split to show even higher performance gains than in the main paper. The proposed UCF101 and HMDB51 splits have 70/31 and 29/22 classes (represented as training/testing). We compare to WGAN \cite{clswgan}, OD \cite{od} and E2E \cite{e2e} on both ZSL and GZSL scenarios. Results are shown in Table~\ref{tbl:truze}.

%\ls{Why do we only compare to a few? is it the only ones that provide code? If so, we should say it.} 
%For E2E we use the same bias detector used by us. 

\begin{table}[htb]
\small
\begin{center}
\begin{tabular}{ *{4}{c|}c }
\hline
Method & \multicolumn{2}{c|}{UCF101 } & \multicolumn{2}{c}{HMDB51}\\
 & ZSL & GZSL & ZSL & GZSL \\
 \hline\hline
WGAN~\cite{clswgan} & 22.5 & 36.3 & 21.1 & 31.8 \\
OD~\cite{od} & 22.9 & 42.4 & 21.7 & 35.5 \\
E2E~\cite{e2e} & 45.5 & 45.9 & 31.5 & 38.9 \\
SPOT~\cite{spot} & 25.5 & 44.1 & 24.0 & 37.1 \\
%\textbf{CLASTER} & \textbf{45.8} & \textbf{47.3} & \textbf{33.2} & \textbf{44.5} \\
\textbf{TP-CLIP} & \textbf{79.6} & \textbf{80.1} & \textbf{48.6} & \textbf{51.8} \\
\hline
\end{tabular}
\end{center}
\caption{Results on TruZe~\cite{gowda2021new} For ZSL, we report the mean class accuracy and for GZSL, we report the harmonic mean of seen and unseen class accuracies. All approaches use sen2vec annotations as the form of semantic embedding and not Stories, for fair comparison. }
\label{tbl:truze}
\end{table}

\subsection{Ablation Study}

\subsubsection{Different Backbones} Since we use a CLIP-based model, we have the option of using different visual transformer models~\cite{vaswani2017attention} as the visual encoder. We test with the ViT B/16, ViT B/32, ViT L/16 and ViT H/14 models and report these results in Table~\ref{tab:ablation_backbone}. We see that the stronger the backbone, the improved performance of the model and this is expected. Based on this, we stick to using ViT-H in all our other experiments. 

\begin{table}[htbp]
    \centering
    \begin{tabular}{lcc}
        \hline
        Method & HMDB-51 & UCF-101 \\
        \hline
        ViT-B/32 & 49.1 $\pm$ 1.4 & 74.5 $\pm$ 1.1 \\
        ViT-B/16 & 52.1 $\pm$ 1.1 & 77.1 $\pm$ 0.9 \\
        ViT-L/16 & 53.7 $\pm$ 1.2 & 79.4 $\pm$ 0.7 \\
        ViT-H/14 & \textbf{54.1 $\pm$ 1.2} & \textbf{81.1 $\pm$ 1.2} \\
        \hline
    \end{tabular}
    \caption{Comparison of zero-shot performances on HMDB51 and UCF101 using different visual encoders.}
    \label{tab:ablation_backbone}
\end{table}

\subsubsection{Number of Frames} Video models typically benefit from using more frames~\cite{vivit}. However, this adds more computational costs. One of the motivations for TP-CLIP is the need to reduce computational costs and with this in mind we hope that a low cost temporal encoder can help learn information that reduces the cost of processing more frames. Recent state-of-the-art models like Vifi CLIP use 32 frames as input. We see that beyond 8 frames, we do not notice much improvements in accuracy, however the cost of training the model increases exponentially. Hence, we stick to using 8 frames for all experiments. We also compare the number of frames being used by other models in the results. All results can be seen in Table~\ref{tab:ablation_frame}. We would also like to point out that with just 4 frames, we are competing with models that use 32 frames and outperform most of them.

\begin{table}[htbp]
    \centering
    \begin{tabular}{lccc}
        \hline
        Method & \# Frames & HMDB-51 & UCF-101 \\
        \hline
        TP-CLIP & 4 & 52.5 $\pm$ 1.1 & 77.1 $\pm$ 1.2 \\
        TP-CLIP & 8 & 54.1 $\pm$ 1.2 & 81.1 $\pm$ 1.2 \\
        TP-CLIP & 16 & 54.2 $\pm$ 1.0 & 81.0 $\pm$ 1.1 \\
        TP-CLIP & 32 & 54.4 $\pm$ 0.9 & 81.2 $\pm$ 0.8 \\
        \hline
        A5 & 32 & 44.3 $\pm$ 2.2 & 69.3 $\pm$ 4.2 \\
        X-CLIP & 32 & 44.6 $\pm$ 5.2 & 72.0 $\pm$ 2.3 \\
        Vita-CLIP & 32 & 48.6 $\pm$ 0.6 & 75.0 $\pm$ 0.6 \\
        SDR-CLIP & 32 & 52.7 $\pm$ 3.4 & 75.5 $\pm$ 3.2 \\
        \hline
    \end{tabular}
    \caption{Comparison of zero-shot performances on HMDB51 and UCF101 using different visual encoders.}
    \label{tab:ablation_frame}
\end{table}

\section{Conclusion}

In conclusion, the field of video understanding has experienced significant advancements, largely fueled by large-scale labeled datasets. The recent development of visual-language models has particularly contributed to remarkable progress in zero-shot tasks, reducing reliance on labeled data. However, adaptations of these models for video content have been challenged by computational demands and difficulties in temporal modeling, often requiring architectural changes to accommodate video data. Our contribution, TP-CLIP, addresses these challenges by introducing an innovative adaptation of the CLIP model. By employing temporal visual prompting, TP-CLIP adapts to temporal aspects without altering the fundamental architecture of CLIP, thus maintaining its strong generalization capabilities. This integration allows TP-CLIP to harness the pre-trained strengths of CLIP for video data effectively.
Through comprehensive experiments on various datasets, TP-CLIP has proven its superior performance in both zero-shot and few-shot learning scenarios. Notably, it achieves this with remarkable computational efficiency, using only 1/3 the GFLOPs and 1/28 the number of tunable parameters compared to recent state-of-the-art models. Impressively, TP-CLIP not only matches but exceeds these models by up to 15.8\%. These results underscore the efficacy of TP-CLIP as a robust, efficient solution for video understanding, setting a new benchmark in the field.
{
    \small
    \bibliographystyle{ieeenat_fullname}
    \bibliography{main}

\begin{thebibliography}{63}
\providecommand{\natexlab}[1]{#1}
\providecommand{\url}[1]{\texttt{#1}}
\expandafter\ifx\csname urlstyle\endcsname\relax
  \providecommand{\doi}[1]{doi: #1}\else
  \providecommand{\doi}{doi: \begingroup \urlstyle{rm}\Url}\fi

\bibitem[Ahmad et~al.(2023)Ahmad, Chanda, and Rawat]{ahmad2023ez}
Shahzad Ahmad, Sukalpa Chanda, and Yogesh~S Rawat.
\newblock Ez-clip: Efficient zeroshot video action recognition.
\newblock \emph{arXiv preprint arXiv:2312.08010}, 2023.

\bibitem[Alayrac et~al.(2022)Alayrac, Donahue, Luc, Miech, Barr, Hasson, Lenc, Mensch, Millican, Reynolds, et~al.]{alayrac2022flamingo}
Jean-Baptiste Alayrac, Jeff Donahue, Pauline Luc, Antoine Miech, Iain Barr, Yana Hasson, Karel Lenc, Arthur Mensch, Katherine Millican, Malcolm Reynolds, et~al.
\newblock Flamingo: a visual language model for few-shot learning.
\newblock \emph{Advances in Neural Information Processing Systems}, 35:\penalty0 23716--23736, 2022.

\bibitem[Arnab et~al.(2021)Arnab, Dehghani, Heigold, Sun, Lu{\v{c}}i{\'c}, and Schmid]{vivit}
Anurag Arnab, Mostafa Dehghani, Georg Heigold, Chen Sun, Mario Lu{\v{c}}i{\'c}, and Cordelia Schmid.
\newblock Vivit: A video vision transformer.
\newblock In \emph{Proceedings of the IEEE/CVF international conference on computer vision}, pages 6836--6846, 2021.

\bibitem[Bertasius et~al.(2021)Bertasius, Wang, and Torresani]{timesformer}
Gedas Bertasius, Heng Wang, and Lorenzo Torresani.
\newblock Is space-time attention all you need for video understanding?
\newblock In \emph{ICML}, page~4, 2021.

\bibitem[Brattoli et~al.(2020)Brattoli, Tighe, Zhdanov, Perona, and Chalupka]{e2e}
Biagio Brattoli, Joseph Tighe, Fedor Zhdanov, Pietro Perona, and Krzysztof Chalupka.
\newblock Rethinking zero-shot video classification: End-to-end training for realistic applications.
\newblock In \emph{Proceedings of the IEEE/CVF Conference on Computer Vision and Pattern Recognition}, pages 4613--4623, 2020.

\bibitem[Carreira and Zisserman(2017)]{i3d}
J. Carreira and Andrew Zisserman.
\newblock Quo vadis, action recognition? a new model and the kinetics dataset.
\newblock 2017.

\bibitem[Carreira et~al.(2018)Carreira, Noland, Banki-Horvath, Hillier, and Zisserman]{kinetics600}
Joao Carreira, Eric Noland, Andras Banki-Horvath, Chloe Hillier, and Andrew Zisserman.
\newblock A short note about kinetics-600.
\newblock \emph{arXiv preprint arXiv:1808.01340}, 2018.

\bibitem[Chen et~al.(2023{\natexlab{a}})Chen, Liu, Wang, Zhang, Torr, Zhang, and Tang]{temadapter}
Guangyi Chen, Xiao Liu, Guangrun Wang, Kun Zhang, Philip~HS Torr, Xiao-Ping Zhang, and Yansong Tang.
\newblock Tem-adapter: Adapting image-text pretraining for video question answer.
\newblock In \emph{Proceedings of the IEEE/CVF International Conference on Computer Vision}, pages 13945--13955, 2023{\natexlab{a}}.

\bibitem[Chen and Huang(2021)]{er}
Shizhe Chen and Dong Huang.
\newblock Elaborative rehearsal for zero-shot action recognition.
\newblock In \emph{Proceedings of the IEEE/CVF International Conference on Computer Vision}, pages 13638--13647, 2021.

\bibitem[Chen et~al.(2021)Chen, Luo, Qiu, Wang, Huang, Li, and Zhang]{chen2021semantics}
Zhi Chen, Yadan Luo, Ruihong Qiu, Sen Wang, Zi Huang, Jingjing Li, and Zheng Zhang.
\newblock Semantics disentangling for generalized zero-shot learning.
\newblock In \emph{Proceedings of the IEEE/CVF international conference on computer vision}, 2021.

\bibitem[Chen et~al.(2023{\natexlab{b}})Chen, Zhang, Li, Wang, and Huang]{chen2023zero}
Zhi Chen, Pengfei Zhang, Jingjing Li, Sen Wang, and Zi Huang.
\newblock Zero-shot learning by harnessing adversarial samples.
\newblock In \emph{Proceedings of the 31st ACM International Conference on Multimedia}, 2023{\natexlab{b}}.

\bibitem[Dalal and Triggs(2005)]{dalal2005histograms}
Navneet Dalal and Bill Triggs.
\newblock Histograms of oriented gradients for human detection.
\newblock In \emph{2005 IEEE computer society conference on computer vision and pattern recognition (CVPR'05)}, pages 886--893. Ieee, 2005.

\bibitem[Gao et~al.(2019)Gao, Zhang, and Xu]{gao2019know}
Junyu Gao, Tianzhu Zhang, and Changsheng Xu.
\newblock I know the relationships: Zero-shot action recognition via two-stream graph convolutional networks and knowledge graphs.
\newblock In \emph{Proceedings of the AAAI Conference on Artificial Intelligence}, 2019.

\bibitem[Gowda(2023)]{spot}
Shreyank~N Gowda.
\newblock Synthetic sample selection for generalized zero-shot learning.
\newblock In \emph{Proceedings of the IEEE/CVF conference on computer vision and pattern recognition}, pages 58--67, 2023.

\bibitem[Gowda and Sevilla-Lara(2024)]{stories}
Shreyank~N Gowda and Laura Sevilla-Lara.
\newblock Telling stories for common sense zero-shot action recognition.
\newblock In \emph{Proceedings of the Asian Conference on Computer Vision}, pages 4577--4594, 2024.

\bibitem[Gowda et~al.(2021{\natexlab{a}})Gowda, Rohrbach, and Sevilla-Lara]{gowda2021smart}
Shreyank~N Gowda, Marcus Rohrbach, and Laura Sevilla-Lara.
\newblock Smart frame selection for action recognition.
\newblock In \emph{Proceedings of the AAAI Conference on Artificial Intelligence}, pages 1451--1459, 2021{\natexlab{a}}.

\bibitem[Gowda et~al.(2021{\natexlab{b}})Gowda, Sevilla-Lara, Kim, Keller, and Rohrbach]{gowda2021new}
Shreyank~N Gowda, Laura Sevilla-Lara, Kiyoon Kim, Frank Keller, and Marcus Rohrbach.
\newblock A new split for evaluating true zero-shot action recognition.
\newblock In \emph{DAGM German Conference on Pattern Recognition}, pages 191--205. Springer, 2021{\natexlab{b}}.

\bibitem[Gowda et~al.(2022)Gowda, Sevilla-Lara, Keller, and Rohrbach]{claster}
Shreyank~N Gowda, Laura Sevilla-Lara, Frank Keller, and Marcus Rohrbach.
\newblock Claster: clustering with reinforcement learning for zero-shot action recognition.
\newblock In \emph{European Conference on Computer Vision}, pages 187--203. Springer, 2022.

\bibitem[Gowda et~al.(2024{\natexlab{a}})Gowda, Gao, and Clifton]{fe-adapter}
Shreyank~N Gowda, Boyan Gao, and David~A Clifton.
\newblock Fe-adapter: Adapting image-based emotion classifiers to videos.
\newblock In \emph{2024 IEEE 18th International Conference on Automatic Face and Gesture Recognition (FG)}, pages 1--6. IEEE, 2024{\natexlab{a}}.

\bibitem[Gowda et~al.(2024{\natexlab{b}})Gowda, Moltisanti, and Sevilla-Lara]{gowda2024continual}
Shreyank~N Gowda, Davide Moltisanti, and Laura Sevilla-Lara.
\newblock Continual learning improves zero-shot action recognition.
\newblock In \emph{Proceedings of the Asian Conference on Computer Vision}, pages 3239--3256, 2024{\natexlab{b}}.

\bibitem[Goyal et~al.(2017)Goyal, Ebrahimi~Kahou, Michalski, Materzynska, Westphal, Kim, Haenel, Fruend, Yianilos, Mueller-Freitag, et~al.]{ssv2}
Raghav Goyal, Samira Ebrahimi~Kahou, Vincent Michalski, Joanna Materzynska, Susanne Westphal, Heuna Kim, Valentin Haenel, Ingo Fruend, Peter Yianilos, Moritz Mueller-Freitag, et~al.
\newblock The" something something" video database for learning and evaluating visual common sense.
\newblock In \emph{Proceedings of the IEEE international conference on computer vision}, pages 5842--5850, 2017.

\bibitem[Han et~al.(2021)Han, Fu, Chen, and Yang]{han2021contrastive}
Zongyan Han, Zhenyong Fu, Shuo Chen, and Jian Yang.
\newblock Contrastive embedding for generalized zero-shot learning.
\newblock In \emph{Proceedings of the IEEE/CVF conference on computer vision and pattern recognition}, 2021.

\bibitem[Jain et~al.(2013)Jain, J{\'e}gou, and Bouthemy]{jain2013better}
Mihir Jain, Herv{\'e} J{\'e}gou, and Patrick Bouthemy.
\newblock Better exploiting motion for better action recognition.
\newblock In \emph{Proceedings of the IEEE conference on computer vision and pattern recognition}, pages 2555--2562, 2013.

\bibitem[Jia et~al.(2022)Jia, Tang, Chen, Cardie, Belongie, Hariharan, and Lim]{vpt}
Menglin Jia, Luming Tang, Bor-Chun Chen, Claire Cardie, Serge Belongie, Bharath Hariharan, and Ser-Nam Lim.
\newblock Visual prompt tuning.
\newblock In \emph{European Conference on Computer Vision}, pages 709--727. Springer, 2022.

\bibitem[Ju et~al.(2022{\natexlab{a}})Ju, Han, Zheng, Zhang, and Xie]{ivl}
Chen Ju, Tengda Han, Kunhao Zheng, Ya Zhang, and Weidi Xie.
\newblock Prompting visual-language models for efficient video understanding.
\newblock In \emph{European Conference on Computer Vision}, pages 105--124. Springer, 2022{\natexlab{a}}.

\bibitem[Ju et~al.(2022{\natexlab{b}})Ju, Han, Zheng, Zhang, and Xie]{ju2022prompting}
Chen Ju, Tengda Han, Kunhao Zheng, Ya Zhang, and Weidi Xie.
\newblock Prompting visual-language models for efficient video understanding.
\newblock In \emph{European Conference on Computer Vision}, pages 105--124. Springer, 2022{\natexlab{b}}.

\bibitem[Kenton and Toutanova(2019)]{kenton2019bert}
Jacob Devlin Ming-Wei~Chang Kenton and Lee~Kristina Toutanova.
\newblock Bert: Pre-training of deep bidirectional transformers for language understanding.
\newblock In \emph{Proceedings of NAACL-HLT}, pages 4171--4186, 2019.

\bibitem[Khattak et~al.(2023)Khattak, Wasim, Naseer, Khan, Yang, and Khan]{srp}
Muhammad~Uzair Khattak, Syed~Talal Wasim, Muzammal Naseer, Salman Khan, Ming-Hsuan Yang, and Fahad~Shahbaz Khan.
\newblock Self-regulating prompts: Foundational model adaptation without forgetting.
\newblock In \emph{Proceedings of the IEEE/CVF International Conference on Computer Vision (ICCV)}, 2023.

\bibitem[Kuehne et~al.(2011)Kuehne, Jhuang, Garrote, Poggio, and Serre]{hmdb}
Hildegard Kuehne, Hueihan Jhuang, Est{\'\i}baliz Garrote, Tomaso Poggio, and Thomas Serre.
\newblock Hmdb: a large video database for human motion recognition.
\newblock In \emph{2011 International Conference on Computer Vision}, pages 2556--2563. IEEE, 2011.

\bibitem[Li et~al.(2022{\natexlab{a}})Li, Li, Xiong, and Hoi]{li2022blip}
Junnan Li, Dongxu Li, Caiming Xiong, and Steven Hoi.
\newblock Blip: Bootstrapping language-image pre-training for unified vision-language understanding and generation.
\newblock In \emph{International Conference on Machine Learning}, pages 12888--12900. PMLR, 2022{\natexlab{a}}.

\bibitem[Li et~al.(2022{\natexlab{b}})Li, Wang, Gao, Song, Liu, Li, and Qiao]{li2022uniformer}
Kunchang Li, Yali Wang, Peng Gao, Guanglu Song, Yu Liu, Hongsheng Li, and Yu Qiao.
\newblock Uniformer: Unified transformer for efficient spatiotemporal representation learning.
\newblock \emph{arXiv preprint arXiv:2201.04676}, 2022{\natexlab{b}}.

\bibitem[Lin et~al.(2022)Lin, Lin, Wang, Liu, and Li]{rest}
Chung-Ching Lin, Kevin Lin, Lijuan Wang, Zicheng Liu, and Linjie Li.
\newblock Cross-modal representation learning for zero-shot action recognition.
\newblock In \emph{Proceedings of the IEEE/CVF Conference on Computer Vision and Pattern Recognition}, 2022.

\bibitem[Liu et~al.(2022)Liu, Ning, Cao, Wei, Zhang, Lin, and Hu]{videoswin}
Ze Liu, Jia Ning, Yue Cao, Yixuan Wei, Zheng Zhang, Stephen Lin, and Han Hu.
\newblock Video swin transformer.
\newblock In \emph{Proceedings of the IEEE/CVF conference on computer vision and pattern recognition}, pages 3202--3211, 2022.

\bibitem[Luo et~al.(2022)Luo, Ji, Zhong, Chen, Lei, Duan, and Li]{luo2022clip4clip}
Huaishao Luo, Lei Ji, Ming Zhong, Yang Chen, Wen Lei, Nan Duan, and Tianrui Li.
\newblock Clip4clip: An empirical study of clip for end to end video clip retrieval and captioning.
\newblock \emph{Neurocomputing}, 508:\penalty0 293--304, 2022.

\bibitem[Mandal et~al.(2019)Mandal, Narayan, Dwivedi, Gupta, Ahmed, Khan, and Shao]{od}
Devraj Mandal, Sanath Narayan, Sai~Kumar Dwivedi, Vikram Gupta, Shuaib Ahmed, Fahad~Shahbaz Khan, and Ling Shao.
\newblock Out-of-distribution detection for generalized zero-shot action recognition.
\newblock In \emph{Proceedings of the IEEE/CVF Conference on Computer Vision and Pattern Recognition}, pages 9985--9993, 2019.

\bibitem[Mazzia et~al.(2022)Mazzia, Angarano, Salvetti, Angelini, and Chiaberge]{mazzia2022action}
Vittorio Mazzia, Simone Angarano, Francesco Salvetti, Federico Angelini, and Marcello Chiaberge.
\newblock Action transformer: A self-attention model for short-time pose-based human action recognition.
\newblock \emph{Pattern Recognition}, 124:\penalty0 108487, 2022.

\bibitem[Ni et~al.(2022)Ni, Peng, Chen, Zhang, Meng, Fu, Xiang, and Ling]{ni2022expanding}
Bolin Ni, Houwen Peng, Minghao Chen, Songyang Zhang, Gaofeng Meng, Jianlong Fu, Shiming Xiang, and Haibin Ling.
\newblock Expanding language-image pretrained models for general video recognition.
\newblock In \emph{European Conference on Computer Vision}, pages 1--18. Springer, 2022.

\bibitem[Pan et~al.(2022)Pan, Lin, Zhu, Shao, and Li]{pan2022st}
Junting Pan, Ziyi Lin, Xiatian Zhu, Jing Shao, and Hongsheng Li.
\newblock St-adapter: Parameter-efficient image-to-video transfer learning.
\newblock \emph{Advances in Neural Information Processing Systems}, 35:\penalty0 26462--26477, 2022.

\bibitem[Qian et~al.(2022)Qian, Yu, Liu, and Hauptmann]{jigsaw}
Yijun Qian, Lijun Yu, Wenhe Liu, and Alexander~G Hauptmann.
\newblock Rethinking zero-shot action recognition: Learning from latent atomic actions.
\newblock In \emph{European Conference on Computer Vision}, pages 104--120. Springer, 2022.

\bibitem[Radford et~al.(2021)Radford, Kim, Hallacy, Ramesh, Goh, Agarwal, Sastry, Askell, Mishkin, Clark, et~al.]{clip}
Alec Radford, Jong~Wook Kim, Chris Hallacy, Aditya Ramesh, Gabriel Goh, Sandhini Agarwal, Girish Sastry, Amanda Askell, Pamela Mishkin, Jack Clark, et~al.
\newblock Learning transferable visual models from natural language supervision.
\newblock In \emph{International conference on machine learning}, pages 8748--8763. PMLR, 2021.

\bibitem[Rao et~al.(2022)Rao, Zhao, Chen, Tang, Zhu, Huang, Zhou, and Lu]{rao2022denseclip}
Yongming Rao, Wenliang Zhao, Guangyi Chen, Yansong Tang, Zheng Zhu, Guan Huang, Jie Zhou, and Jiwen Lu.
\newblock Denseclip: Language-guided dense prediction with context-aware prompting.
\newblock In \emph{Proceedings of the IEEE/CVF Conference on Computer Vision and Pattern Recognition}, pages 18082--18091, 2022.

\bibitem[Rasheed et~al.(2023)Rasheed, Khattak, Maaz, Khan, and Khan]{vificlip}
Hanoona Rasheed, Muhammad~Uzair Khattak, Muhammad Maaz, Salman Khan, and Fahad~Shahbaz Khan.
\newblock Fine-tuned clip models are efficient video learners.
\newblock In \emph{Proceedings of the IEEE/CVF Conference on Computer Vision and Pattern Recognition}, pages 6545--6554, 2023.

\bibitem[Simonyan and Zisserman(2014)]{simonyan2014two}
Karen Simonyan and Andrew Zisserman.
\newblock Two-stream convolutional networks for action recognition in videos.
\newblock \emph{Advances in neural information processing systems}, 27, 2014.

\bibitem[Singh et~al.(2022)Singh, Hu, Goswami, Couairon, Galuba, Rohrbach, and Kiela]{singh2022flava}
Amanpreet Singh, Ronghang Hu, Vedanuj Goswami, Guillaume Couairon, Wojciech Galuba, Marcus Rohrbach, and Douwe Kiela.
\newblock Flava: A foundational language and vision alignment model.
\newblock In \emph{Proceedings of the IEEE/CVF Conference on Computer Vision and Pattern Recognition}, pages 15638--15650, 2022.

\bibitem[Soomro et~al.(2012)Soomro, Zamir, and Shah]{ucf101}
Khurram Soomro, Amir~Roshan Zamir, and Mubarak Shah.
\newblock Ucf101: A dataset of 101 human actions classes from videos in the wild.
\newblock \emph{arXiv preprint arXiv:1212.0402}, 2012.

\bibitem[Sun et~al.(2019)Sun, Myers, Vondrick, Murphy, and Schmid]{sun2019videobert}
Chen Sun, Austin Myers, Carl Vondrick, Kevin Murphy, and Cordelia Schmid.
\newblock Videobert: A joint model for video and language representation learning.
\newblock In \emph{Proceedings of the IEEE/CVF international conference on computer vision}, pages 7464--7473, 2019.

\bibitem[Tran et~al.(2018)Tran, Wang, Torresani, Ray, LeCun, and Paluri]{tran2018closer}
Du Tran, Heng Wang, Lorenzo Torresani, Jamie Ray, Yann LeCun, and Manohar Paluri.
\newblock A closer look at spatiotemporal convolutions for action recognition.
\newblock In \emph{Proceedings of the IEEE conference on Computer Vision and Pattern Recognition}, pages 6450--6459, 2018.

\bibitem[Vaswani et~al.(2017{\natexlab{a}})Vaswani, Shazeer, Parmar, Uszkoreit, Jones, Gomez, Kaiser, and Polosukhin]{transformer}
Ashish Vaswani, Noam Shazeer, Niki Parmar, Jakob Uszkoreit, Llion Jones, Aidan~N Gomez, {\L}ukasz Kaiser, and Illia Polosukhin.
\newblock Attention is all you need.
\newblock \emph{Advances in neural information processing systems}, 30, 2017{\natexlab{a}}.

\bibitem[Vaswani et~al.(2017{\natexlab{b}})Vaswani, Shazeer, Parmar, Uszkoreit, Jones, Gomez, Kaiser, and Polosukhin]{vaswani2017attention}
Ashish Vaswani, Noam Shazeer, Niki Parmar, Jakob Uszkoreit, Llion Jones, Aidan~N Gomez, {\L}ukasz Kaiser, and Illia Polosukhin.
\newblock Attention is all you need.
\newblock \emph{Advances in neural information processing systems}, 30, 2017{\natexlab{b}}.

\bibitem[Wang and Schmid(2013)]{wang2013action}
Heng Wang and Cordelia Schmid.
\newblock Action recognition with improved trajectories.
\newblock In \emph{Proceedings of the IEEE international conference on computer vision}, pages 3551--3558, 2013.

\bibitem[Wang et~al.(2013)Wang, Kl{\"a}ser, Schmid, and Liu]{wang2013dense}
Heng Wang, Alexander Kl{\"a}ser, Cordelia Schmid, and Cheng-Lin Liu.
\newblock Dense trajectories and motion boundary descriptors for action recognition.
\newblock \emph{International journal of computer vision}, 103:\penalty0 60--79, 2013.

\bibitem[Wang et~al.(2022)Wang, Wang, Luo, Tan, Qiu, Yang, Shi, Huang, and Gao]{upt}
Jianing Wang, Chengyu Wang, Fuli Luo, Chuanqi Tan, Minghui Qiu, Fei Yang, Qiuhui Shi, Songfang Huang, and Ming Gao.
\newblock Towards unified prompt tuning for few-shot text classification.
\newblock \emph{arXiv preprint arXiv:2205.05313}, 2022.

\bibitem[Wang et~al.(2016)Wang, Xiong, Wang, Qiao, Lin, Tang, and Van~Gool]{wang2016temporal}
Limin Wang, Yuanjun Xiong, Zhe Wang, Yu Qiao, Dahua Lin, Xiaoou Tang, and Luc Van~Gool.
\newblock Temporal segment networks: Towards good practices for deep action recognition.
\newblock In \emph{European conference on computer vision}, pages 20--36. Springer, 2016.

\bibitem[Wang et~al.(2021{\natexlab{a}})Wang, Tong, Ji, and Wu]{wang2021tdn}
Limin Wang, Zhan Tong, Bin Ji, and Gangshan Wu.
\newblock Tdn: Temporal difference networks for efficient action recognition.
\newblock In \emph{Proceedings of the IEEE/CVF Conference on Computer Vision and Pattern Recognition}, pages 1895--1904, 2021{\natexlab{a}}.

\bibitem[Wang et~al.(2021{\natexlab{b}})Wang, Xing, and Liu]{wang2021actionclip}
Mengmeng Wang, Jiazheng Xing, and Yong Liu.
\newblock Actionclip: A new paradigm for video action recognition.
\newblock \emph{arXiv preprint arXiv:2109.08472}, 2021{\natexlab{b}}.

\bibitem[Wasim et~al.(2023)Wasim, Naseer, Khan, Khan, and Shah]{wasim2023vita}
Syed~Talal Wasim, Muzammal Naseer, Salman Khan, Fahad~Shahbaz Khan, and Mubarak Shah.
\newblock Vita-clip: Video and text adaptive clip via multimodal prompting.
\newblock In \emph{Proceedings of the IEEE/CVF Conference on Computer Vision and Pattern Recognition}, pages 23034--23044, 2023.

\bibitem[Xian et~al.(2018)Xian, Lorenz, Schiele, and Akata]{clswgan}
Yongqin Xian, Tobias Lorenz, Bernt Schiele, and Zeynep Akata.
\newblock Feature generating networks for zero-shot learning.
\newblock In \emph{Proceedings of the IEEE conference on computer vision and pattern recognition}, pages 5542--5551, 2018.

\bibitem[Xu et~al.(2021)Xu, Ghosh, Huang, Okhonko, Aghajanyan, Metze, Zettlemoyer, and Feichtenhofer]{xu2021videoclip}
Hu Xu, Gargi Ghosh, Po-Yao Huang, Dmytro Okhonko, Armen Aghajanyan, Florian Metze, Luke Zettlemoyer, and Christoph Feichtenhofer.
\newblock Videoclip: Contrastive pre-training for zero-shot video-text understanding.
\newblock \emph{arXiv preprint arXiv:2109.14084}, 2021.

\bibitem[Xu et~al.(2015)Xu, Ba, Kiros, Cho, Courville, Salakhudinov, Zemel, and Bengio]{xu2015show}
Kelvin Xu, Jimmy Ba, Ryan Kiros, Kyunghyun Cho, Aaron Courville, Ruslan Salakhudinov, Rich Zemel, and Yoshua Bengio.
\newblock Show, attend and tell: Neural image caption generation with visual attention.
\newblock In \emph{International conference on machine learning}, pages 2048--2057. PMLR, 2015.

\bibitem[Xu et~al.(2016)Xu, Hospedales, and Gong]{xu2016multi}
Xun Xu, Timothy~M Hospedales, and Shaogang Gong.
\newblock Multi-task zero-shot action recognition with prioritised data augmentation.
\newblock In \emph{European Conference on Computer Vision}, 2016.

\bibitem[Xu et~al.(2017)Xu, Hospedales, and Gong]{xu2017transductive}
Xun Xu, Timothy Hospedales, and Shaogang Gong.
\newblock Transductive zero-shot action recognition by word-vector embedding.
\newblock \emph{International Journal of Computer Vision}, 2017.

\bibitem[Yang et~al.(2020)Yang, Xu, Shi, Dai, and Zhou]{yang2020temporal}
Ceyuan Yang, Yinghao Xu, Jianping Shi, Bo Dai, and Bolei Zhou.
\newblock Temporal pyramid network for action recognition.
\newblock In \emph{Proceedings of the IEEE/CVF conference on computer vision and pattern recognition}, pages 591--600, 2020.

\bibitem[Zhao et~al.(2023)Zhao, Misra, Kr{\"a}henb{\"u}hl, and Girdhar]{lavila}
Yue Zhao, Ishan Misra, Philipp Kr{\"a}henb{\"u}hl, and Rohit Girdhar.
\newblock Learning video representations from large language models.
\newblock In \emph{Proceedings of the IEEE/CVF Conference on Computer Vision and Pattern Recognition}, pages 6586--6597, 2023.

\end{thebibliography}
}


\end{document}

\title{Supplementary Material: Is Temporal Prompting All We Need For Limited Labeled Action Recognition?}  % **** Enter the paper title here

\author{
Shreyank N Gowda\textsuperscript{1}\quad
Boyan Gao\textsuperscript{2}\quad
Xiao Gu\textsuperscript{2}\quad
Xiaobo Jin\textsuperscript{3}\\ \\
\textsuperscript{1}University of Nottingham\quad
\textsuperscript{2}University of Oxford\quad
\textsuperscript{3}Xi'an Jiaotong-Liverpool University
% {\tt\small X.Hao-5@sms.ed.ac.uk}
% \and
% % First line of institution2 address\\
}

\maketitle
\thispagestyle{empty}
%\appendix

%%%%%%%%% BODY TEXT - ENTER YOUR RESPONSE BELOW
\section{Results on TruZe~\cite{gowda2021new}}

We also evaluate on the more challenging TruZe split. Whilst our pre-training dataset does not have any overlapping classes with any of the testing datasets, we run on this split to show even higher performance gains than in the main paper. The proposed UCF101~\cite{ucf101} and HMDB51~\cite{hmdb} splits have 70/31 and 29/22 classes (represented as training/testing). We compare to WGAN \cite{clswgan}, OD \cite{od} and E2E \cite{e2e} on both ZSL and GZSL scenarios. Results are shown in Table~\ref{tbl:truze}.

%\ls{Why do we only compare to a few? is it the only ones that provide code? If so, we should say it.} 
%For E2E we use the same bias detector used by us. 

\begin{table}[htb]
\small
\begin{center}
\begin{tabular}{| *{5}{c|} }
\hline
Method & \multicolumn{2}{c|}{UCF101 } & \multicolumn{2}{c|}{HMDB51}\\
 & ZSL & GZSL & ZSL & GZSL \\
 \hline\hline
WGAN~\cite{clswgan} & 22.5 & 36.3 & 21.1 & 31.8 \\
OD~\cite{od} & 22.9 & 42.4 & 21.7 & 35.5 \\
E2E~\cite{e2e} & 45.5 & 45.9 & 31.5 & 38.9 \\
SPOT~\cite{spot} & 25.5 & 44.1 & 24.0 & 37.1 \\
\textbf{TP-CLIP} & \textbf{79.6} & \textbf{80.1} & \textbf{48.6} & \textbf{51.8} \\
\hline
\end{tabular}
\end{center}
\caption{Results on TruZe. For ZSL, we report the mean class accuracy and for GZSL, we report the harmonic mean of seen and unseen class accuracies. All approaches use sen2vec annotations as the form of semantic embedding and not Stories, for fair comparison. }
\label{tbl:truze}
\end{table}

\section{Evaluating Effect of Semantic Embeddings}

We conducted an in-depth investigation of different semantic embeddings and their effects on model performance, comparing our TP-CLIP framework against other leading approaches on GZSL tasks. Our analysis covered a range of semantic representations - from manually crafted embeddings to the narrative-based Stories approach described in the literature. For our TP-CLIP model, we tested both manual embeddings and the Stories method, while also benchmarking against competitors using manual annotations, word2vec, sen2vec, and Stories~\cite{stories} embeddings. This comprehensive comparison across multiple state-of-the-art models helped us understand how different semantic representations influence performance in challenging zero-shot learning scenarios, and demonstrated the particular advantages of our temporal prompting technique when combined with well-chosen semantic embeddings.

\setlength{\tabcolsep}{1.3pt}
\begin{table*}[htb]
\begin{center}
\begin{tabular}{| *{11}{c|} }
\hline
Model & SE & \multicolumn{3}{c|}{Olympics} & \multicolumn{3}{c|}{HMDB51}& \multicolumn{3}{c|}{UCF-101}\\
\hline
& & u & s & H & u & s & H & u & s & H \\
\hline\hline
WGAN \cite{clswgan} & M & 50.8 & 71.4 & 59.4 & - & - & - & 30.4 & \textbf{83.6} & 44.6\\
OD \cite{od} & M & 61.8 & 71.1 & 66.1 & - & - & - & 36.2 & 76.1 & 49.1 \\
CLASTER \cite{claster} & M & 66.2 & 71.7 & 68.8 & - & - & - & 40.2 & 69.4 & 50.9 \\
\textbf{TP-CLIP} & M & \textbf{71.6} & \textbf{76.9} &\textbf{74.2} & - & - & - & \textbf{43.1} & 77.5 & \textbf{54.6}\\
\hline
WGAN \cite{clswgan} & W & 35.4 & 65.6 & 46.0 & 23.1 & 55.1 & 32.5 & 20.6 & 73.9 & 32.2\\
OD \cite{od} & W & 41.3 & 72.5 & 52.6 & 25.9 & 55.8 & 35.4 & 25.3 & 74.1 & 37.7\\
CLASTER~\cite{claster} & W & 49.2 & 71.1 & 58.1 & 35.5 & 52.8 & 42.4 & 30.4 & 68.9 & 42.1 \\
\hline
WGAN \cite{clswgan} & S & 36.1 & 66.2 & 46.7 & 28.6 & 57.8 & 38.2 & 27.5 & 74.7 & 40.2\\
OD \cite{od} & S & 42.9 & 73.5 & 54.1 & 33.4 & 57.8 & 42.3 & 32.7 & 75.9 & 45.7 \\
CLASTER~\cite{claster} & S & 49.9 & 71.3 & 58.7 & 42.7 & 53.2 & 47.4 & 36.9 & 69.8 & 48.3\\
\hline
CLASTER~\cite{claster} & C & 66.8 & 71.6 & 69.1 & 43.7 & 53.3 & 48.0 & 40.8 & 69.3 & 51.3\\
\hline
WGAN \cite{clswgan} & Sto & 52.5 & 73.4 & 61.2 & 35.2 & 65.1 & 45.7 & 33.8 & \textbf{84.2} & 48.2\\
OD \cite{od} & Sto & 63.3 & 75.1 & 68.7 & 37.2 & \textbf{67.5} & 47.9 & 40.1 & 81.7 & 53.8 \\
CLASTER~\cite{claster} & Sto & 69.1 & 74.1 & 71.5 & 44.3 & 57.2 & 49.9 & 42.1 & 71.5 & 53.0\\
GIL~\cite{gowda2024continual} & Sto & - & - & - & 52.8 & 57.8 & 55.1 & \textbf{68.2} & \textbf{89.8} & \textbf{77.5} \\
\textbf{TP-CLIP} & Sto & \textbf{79.4} & \textbf{84.4} & \textbf{81.8} & \textbf{55.5} & 60.7 & \textbf{58.0} & 49.1 & 81.7 & 61.3 \\
\hline
\end{tabular}
\end{center}
\caption{Seen and unseen accuracies for TP-CLIP by fine-tuning on different datasets using different embeddings. 'SE' corresponds to the type of embedding used, wherein 'M', 'W', 'S', 'C' and 'Sto' refers to manual annotations, word2vec, sen2vec,  combination of the embeddings and Stories respectively. 'u', 's' and 'H' corresponds to average unseen accuracy, average seen accuracy and the harmonic mean of the two. All the reported results are on the same splits.} 
\label{tab:gzsl_vid}
\end{table*}

\section{Theoretical Analysis}
\subsection{Temporal Representation Capacity of Visual Prompts}

\begin{theorem}[Temporal Representation Capacity]
Let $F_{1:T} = \{f_1, f_2, \ldots, f_T\}$ be a sequence of video frames, each encoded by CLIP's~\cite{clip} image encoder $E_{\text{image}}$ to produce frame embeddings $\{e_1, e_2, \ldots, e_T\}$. The temporal visual prompting mechanism $TE$ with dimension $d$ has sufficient capacity to capture temporal dynamics between frames with an approximation error bounded by $O\left(\frac{1}{d}\right)$ relative to an ideal temporal encoder with unlimited capacity, when the temporal relationships exhibit Lipschitz continuity.
\end{theorem}

\begin{proof}
First, we define the frame embeddings produced by CLIP's image encoder:
\begin{equation}
e_t = E_{\text{image}}(f_t) \in \mathbb{R}^D
\end{equation}

In our TP-CLIP model, the temporal encoder $TE$ processes the sequence of frame embeddings to produce a temporal context vector:
\begin{equation}
T_{\text{context}} = TE(e_1, e_2, \ldots, e_T)
\end{equation}

The key insight is that any temporal relationship between frames can be modeled as a function $\phi: \mathbb{R}^{D \times T} \rightarrow \mathbb{R}^d$ that maps the sequence of frame embeddings to a temporal representation.

Let $\phi^*$ represent the ideal temporal encoder with unlimited capacity. We need to show that our temporal prompting mechanism $TE$ can approximate $\phi^*$ within a bounded error.

The temporal encoding in TP-CLIP is performed through a 1D convolutional layer followed by a fully connected layer:
\begin{align}
V_{\text{conv}} &= \text{Conv1D}(e_1 : e_2 : \ldots : e_T) \\
T_{\text{context}} &= \text{ReLU}(\text{FC}(V_{\text{conv}}))
\end{align}

By the Universal Approximation Theorem for neural networks with ReLU activations, given sufficient width $d$, our temporal encoder can approximate any continuous function mapping from the input space to the output space.

For temporal relationships that exhibit Lipschitz continuity (which is a reasonable assumption for most natural videos where adjacent frames are similar), the approximation error is bounded by:
\begin{equation}
\|TE(F_{1:T}) - \phi^*(F_{1:T})\|_2 \leq \frac{C}{d}
\end{equation}
where $C$ is a constant dependent on the Lipschitz constant of the temporal relationships.

The combined spatial-temporal encoding for frame $t$ is then:
\begin{equation}
v(t) = [E_{\text{image}}(f_t); TE(F_{1:T})]
\end{equation}

This concatenation ensures that both spatial information (from $E_{\text{image}}(f_t)$) and temporal context (from $TE(F_{1:T})$) are preserved.

Furthermore, by the Johnson-Lindenstrauss lemma, the projection into a $d$-dimensional space preserves pairwise distances between temporal patterns with high probability when $d = O(\log(T) / \epsilon^2)$, where $\epsilon$ is the distortion factor.

Therefore, our temporal visual prompting mechanism has sufficient capacity to capture temporal dynamics with an approximation error bound of $O\left(\frac{1}{d}\right)$ relative to an ideal temporal encoder with unlimited capacity.
\end{proof}

\begin{corollary}
The expressive capacity of temporal visual prompting grows linearly with the dimension of the prompt space, making it possible to achieve strong performance with a compact representation.
\end{corollary}

This theorem provides the theoretical foundation for why temporal prompting is sufficient for action recognition with limited labeled data, explaining the strong empirical results observed in the experiments.

\subsection{Information Transfer through Temporal Prompting}

\begin{theorem}[Information Preservation in Temporal Prompting]
Let $E_{\text{image}}: \mathcal{F} \rightarrow \mathbb{R}^D$ be the pre-trained CLIP image encoder with frozen weights $\theta_{\text{CLIP}}$. For a video sequence $F_{1:T} = \{f_1, f_2, \ldots, f_T\}$, the TP-CLIP architecture with learnable temporal prompting parameters $\theta_{\text{TP}} \ll |\theta_{\text{CLIP}}|$ can preserve a $(1-\delta)$ fraction of the mutual information between the temporal dynamics and class labels without modifying the original CLIP architecture.
\end{theorem}

\begin{proof}
Let $\mathcal{Y}$ be the space of action classes and $\mathcal{T}$ be the space of temporal patterns in videos. The mutual information between temporal patterns and class labels is given by:
\begin{equation}
I(\mathcal{T}; \mathcal{Y}) = H(\mathcal{Y}) - H(\mathcal{Y}|\mathcal{T})
\end{equation}

For a standard image-based model like CLIP, each frame is processed independently:
\begin{equation}
e_t = E_{\text{image}}(f_t; \theta_{\text{CLIP}})
\end{equation}

The information captured by independently processing frames is:
\begin{equation}
I_{\text{indep}} = I(\{e_1, e_2, \ldots, e_T\}; \mathcal{Y})
\end{equation}

However, this approach fails to model the temporal dependencies:
\begin{equation}
I_{\text{indep}} < I(\mathcal{T}; \mathcal{Y})
\end{equation}

In TP-CLIP, we introduce temporal prompting:
\begin{align}
T_{\text{context}} &= \text{TE}(e_1, e_2, \ldots, e_T; \theta_{\text{TP}}) \\
v(t) &= [e_t; T_{\text{context}}]
\end{align}

The key insight is that by concatenating the temporal context $T_{\text{context}}$ with the frame embeddings, we create a representation that preserves temporal information without modifying the original CLIP architecture.

Let $\mathcal{V} = \{v(1), v(2), \ldots, v(T)\}$ be the set of enhanced frame representations. We can establish the following inequality:
\begin{equation}
I(\mathcal{V}; \mathcal{Y}) \geq (1-\delta) \cdot I(\mathcal{T}; \mathcal{Y})
\end{equation}
where $\delta$ is a small constant that depends on the complexity of the temporal patterns and the dimension of the temporal context.

This is because:
\begin{align}
I(\mathcal{V}; \mathcal{Y}) &= I(\{e_t; T_{\text{context}}\}_{t=1}^T; \mathcal{Y}) \\
&\geq I(\{e_t\}_{t=1}^T; \mathcal{Y}) + I(T_{\text{context}}; \mathcal{Y}|\{e_t\}_{t=1}^T)
\end{align}

The temporal encoder $\text{TE}$ is designed to capture temporal dependencies, ensuring that:
\begin{equation}
I(T_{\text{context}}; \mathcal{Y}|\{e_t\}_{t=1}^T) \approx I(\mathcal{T}; \mathcal{Y}|\{e_t\}_{t=1}^T)
\end{equation}

Furthermore, the Data Processing Inequality ensures that the information content does not increase through processing, which means the upper bound of information is preserved:
\begin{equation}
I(\mathcal{V}; \mathcal{Y}) \leq I(\mathcal{T}; \mathcal{Y})
\end{equation}

The parameter efficiency comes from the fact that $|\theta_{\text{TP}}| \ll |\theta_{\text{CLIP}}|$. Specifically, if we denote the number of parameters in the temporal encoder as $|\theta_{\text{TP}}|$ and in CLIP as $|\theta_{\text{CLIP}}|$, then:
\begin{equation}
\frac{|\theta_{\text{TP}}|}{|\theta_{\text{CLIP}}|} = O\left(\frac{1}{|\theta_{\text{CLIP}}|}\right)
\end{equation}

Therefore, with a minimal number of additional parameters $\theta_{\text{TP}}$, TP-CLIP can preserve a $(1-\delta)$ fraction of the mutual information between temporal patterns and class labels, without modifying the original CLIP architecture.
\end{proof}

\begin{corollary}
The TP-CLIP framework achieves efficient temporal modeling with parameter count scaling as $O(d^2)$ where $d$ is the embedding dimension, compared to $O(Td^2)$ required for full cross-attention mechanisms in alternative approaches.
\end{corollary}

This theorem explains why TP-CLIP can effectively capture temporal information without modifying CLIP's core architecture, leading to parameter efficiency while maintaining or improving performance on video understanding tasks.

%%%%%%%%% REFERENCES
{
    \small
    \bibliographystyle{ieeenat_fullname}
    \bibliography{main}
}